\documentclass{article}

\usepackage[letterpaper,margin=1in]{geometry}
\usepackage{times}
\usepackage{amsmath,amssymb,amsthm}
\usepackage{graphicx}
\usepackage{booktabs}
\usepackage{xcolor}
\usepackage{hyperref}
\usepackage{url}
\usepackage[round]{natbib}
\usepackage{tikz}
\usetikzlibrary{arrows.meta,positioning,fit,backgrounds}
\usepackage{float}
\hypersetup{colorlinks=true, linkcolor=blue, citecolor=blue, urlcolor=blue}

\newcommand{\Reff}{R_{\mathrm{eff}}}
\newcommand{\Reffe}{\tilde{R}_{\mathrm{eff}}}
\newcommand{\Rstable}{R_{\mathrm{stable}}}
\newcommand{\rstar}{r^{*}}
\newcommand{\Attn}{\mathrm{Attn}}
\newcommand{\softmax}{\mathrm{softmax}}

\theoremstyle{plain}
\newtheorem{theorem}{Theorem}

\theoremstyle{definition}

\title{The Entropic Bound for Transformers:\
Why Static Rank Fails and Attention-Native Rank Recovers}

\author{Byeong Hoon Yoon\\Independent Researcher\\\texttt{iam@bhyoon.me}\\ORCID: \href{https://orcid.org/0009-0004-1188-5390}{0009-0004-1188-5390}}

\date{}

\begin{document}
\maketitle

\begin{abstract}
Neural scaling laws describe how loss decreases as models, data, and compute grow, but they do not answer a prior question: for a fixed task, what is the \emph{minimum} model capacity required to solve it? We study this through the \textbf{Entropic Bound}, a spectral notion of task-intrinsic capacity for Transformers. We first prove that, in a linear attention surrogate, the intrinsic rank $\rstar$ of the token-mixing operator is a \emph{tight} lower bound: any rank-deficient model incurs unavoidable excess risk, and the bound is achievable at $\rstar$. We further analyze gradient-descent recovery under standard low-rank implicit-bias assumptions, showing the learned effective rank concentrates at $\rstar$. We confirm all three properties empirically and show $\rstar$ is recoverable from data \emph{before training}. We then ask whether this transfers to real attention. A naive transfer \textbf{fails}, and through a controlled interpolation ladder we localize the cause precisely: it is not softmax and not a rank constraint, but the \textbf{input-conditioned} nature of attention's mixing operator, which a static weight kernel cannot summarize. Motivated by this, we introduce an \textbf{attention-native intrinsic rank}---the minimum query--key kernel rank realizing the task within the attention class---and show that under this definition the full Entropic Bound structure (deficiency, achievability, recovery) is restored for both linear and softmax attention, with the \emph{energy} effective rank as the estimator robust to softmax distortion. Finally, we map the boundary of data-only predictability: $\rstar$ is exactly recoverable for linear QK attention, even without the value map at scale, while softmax attention admits only partial pre-training recovery due to nonlinear inversion and kernel--value identifiability effects. Our results reframe the Entropic Bound from a post-hoc descriptor into an attention-native capacity measure with a precisely characterized predictability frontier.
\end{abstract}

\section{Introduction}

The empirical regularities of large-model training are by now well documented: loss falls as a power law in parameters, data, and compute. Yet scaling laws describe a \emph{trajectory}, not a \emph{floor}. They tell us how performance improves as we add capacity; they do not tell us how little capacity a given task fundamentally requires. This quantity---the minimum model size at which a task becomes solvable to a target error---is conceptually prior to scaling, and it is the object we study.

A useful analogy comes from coding theory. Once an alphabet and admissible code lengths are fixed, only finitely many messages are representable. Likewise, once a model family is fixed, a finite parameter vector realizes only a restricted set of input--output maps. If a task demands many distinguishable predictive modes, any model solving it must carry sufficient \emph{effective} capacity. The difficulty is that raw parameter count is not architecture-invariant, so the right object is a \textbf{task-intrinsic effective dimension} that induces architecture-dependent parameter lower bounds.

For a linear attention surrogate this program can be carried out exactly. We define the intrinsic rank $\rstar$ as the minimum rank of a token-mixing operator that reaches the noise floor, and prove it is tight: rank deficiency forces excess risk (necessity), the floor is achievable at $\rstar$ (sufficiency), and gradient descent from small initialization converges to it (implicit bias). These results give a clean and complete picture---in the surrogate.

The central question of this paper is whether that picture survives contact with real attention. We find that a \emph{naive} transfer fails, but the failure is informative. By constructing an interpolation ladder that begins at the linear surrogate and changes one architectural component at a time, we localize the breakdown to a single cause---and, contrary to the natural guess, that cause is \textbf{not softmax}. It is the input-conditioned bilinear form $X K X^\top$ that attention uses to compute scores: because this operator depends on the input, no static property of the weight kernel $K$ (neither its rank nor its effective rank) determines the model's task capacity. A control with an unconstrained full-rank kernel fails identically, ruling out a rank-constraint explanation.

This motivates a redefinition. Rather than measuring the static kernel, we define an \textbf{attention-native intrinsic rank} as the minimum query--key kernel rank that realizes the task \emph{within the attention class itself}. Generating teacher tasks from this same class removes the operator mismatch by construction, and under the new definition the entire Entropic Bound structure returns: rank-deficient students incur excess risk, students at $\rstar$ reach the floor, and over-parameterized students recover $\rstar$. This holds not only for linear attention but also for softmax attention, provided capacity is measured by the \textbf{energy} effective rank---which we show is robust to the small-singular-value noise softmax introduces.

Finally, we characterize \emph{predictability}: estimating $\rstar$ from data without training. We show exact recovery for linear QK attention, including the V-unknown setting at scale, and identify the remaining frontier in softmax attention, where nonlinear inversion corrupts the numerical rank and leaves only partial energy-rank predictability.

\paragraph{Contributions.}
\begin{enumerate}
\item We restate and empirically validate the tightness of the Entropic Bound in the linear attention surrogate (deficiency, achievability, recovery), and demonstrate data-only recovery of $\rstar$ (Sections~\ref{sec:linear},~\ref{sec:native}).
\item We show that a naive transfer to real attention fails, and via a controlled interpolation ladder localize the cause to \textbf{input conditioning}---not softmax, not a rank constraint (Section~\ref{sec:break}).
\item We introduce the \textbf{attention-native intrinsic rank} and show it restores the full bound structure for linear and softmax attention, with the energy effective rank as the softmax-robust estimator (Section~\ref{sec:native}).
\item We map the \textbf{predictability frontier}: exact data-only recovery for linear QK attention, including the V-unknown setting at scale, and partial recovery under softmax, where nonlinear inversion corrupts the numerical rank and energy-rank estimators remain imperfect (Section~\ref{sec:predict}).
\end{enumerate}

\begin{figure}[H]
\centering
\begin{tikzpicture}[
  node distance=6mm,
  box/.style={rectangle, rounded corners, draw, thick, align=center,
              text width=0.27\linewidth, inner sep=4pt, minimum height=20mm,
              font=\small},
  ok/.style={box, draw=green!55!black, fill=green!6},
  bad/.style={box, draw=red!65!black, fill=red!6},
  rec/.style={box, draw=blue!60!black, fill=blue!6},
  arr/.style={-{Latex[length=2.5mm]}, thick}
]
\node[ok] (b0) {\textbf{Fixed-operator surrogate}\\[2pt]
  $Y=A^{*}XV^{*}$\\[2pt]
  B0 succeeds; $\rstar$ tight\\ (Thm 1--4, Table~\ref{tab:phase0})};
\node[bad, right=of b0] (b1) {\textbf{Static-rank transfer}\\[2pt]
  fixed target, student $XKX^{\top}$\\[2pt]
  B1/B1c/B2 fail\\ \emph{input-conditioning} (Table~\ref{tab:ladder})};
\node[rec, right=of b1] (an) {\textbf{Attention-native teacher}\\[2pt]
  $Y=\Attn_{K^{*},V^{*}}(X)$, $\mathrm{rank}(K^{*}){=}\rstar$\\[2pt]
  deficiency/achievability/\\recovery restored (Table~\ref{tab:grid})};
\draw[arr] (b0) -- node[above, font=\scriptsize] {naive} node[below, font=\scriptsize] {transfer} (b1);
\draw[arr] (b1) -- node[above, font=\scriptsize] {redefine in} node[below, font=\scriptsize] {attn.\ class} (an);
\end{tikzpicture}
\caption{Overview of the paper's logic. The fixed-operator linear surrogate admits
a tight Entropic Bound (left). A naive transfer to attention fails even with a
full-rank kernel (center), localizing the break to input-conditioned mixing rather
than softmax or rank deficiency. Defining the intrinsic rank inside the attention
class restores deficiency, achievability, and recovery (right).}
\label{fig:overview}
\end{figure}
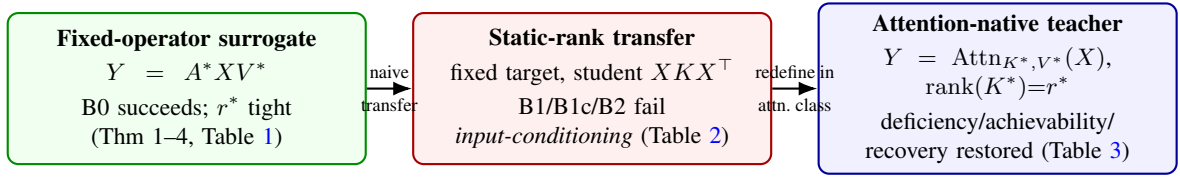

\section{Related Work}

\paragraph{Neural scaling laws.} Empirical scaling laws characterize loss as a smooth power law in model size, data, and compute \citep{kaplan2020scaling, hoffmann2022chinchilla}. Our framework is complementary: rather than the trajectory of improvement, we study the \emph{critical capacity threshold} a task imposes.

\paragraph{Information-theoretic lower bounds.} Rate--distortion and predictive-information arguments bound representation size from below \citep{cover2006elements, berger1971ratedistortion}. Our spectral pruning view is a matrix-valued analogue: the rate is the retained rank $r$ and the distortion is the tail energy $\sum_{j>r}\sigma_j^2$; the effective rank marks where an additional spectral mode stops paying for itself.

\paragraph{Statistical learning theory.} Classical complexity measures (VC dimension, Rademacher complexity, covering numbers) are worst-case and distribution-free, and too coarse here \citep{vapnik1998statistical, bartlett2002rademacher}. Our analysis is distribution-dependent: metric entropy controls generalization, whereas our bound controls \emph{representational capacity}.

\paragraph{Low-rank structure and compression.} A large literature shows fine-tuning updates and trained weights are low-rank, and that low-rank parameterizations match dense models cheaply \citep{hu2022lora, aghajanyan2021intrinsic}. These address the rank of updates or compressibility of trained models; we ask the prior question of the minimum rank to realize a task, and show the \emph{static kernel rank} is the wrong object for attention.

\paragraph{Implicit regularization and spectral bias.} Gradient descent on over-parameterized models is biased toward low-rank solutions \citep{srebro2004maximum, gunasekar2017implicit, arora2019implicit}. This underlies why a trained kernel's effective rank can converge to the task-intrinsic rank, and grounds our recovery experiments.

\paragraph{Effective rank and attention rank collapse.} The effective rank \citep{roy2007effective} and the phenomenon of attention rank collapse \citep{dong2021attention, kobayashi2020attention} have been studied as dynamics and signal-propagation effects. Our contribution is orthogonal: we tie a spectral capacity measure to a \emph{task-intrinsic} lower bound, identify why the static-kernel measure fails, and propose the realized-operator alternative.

\section{Preliminaries}

\paragraph{Population risk.} For a distribution $D$ over $(x,y)$ and loss $\ell$, the population risk is $L_D(f)=\mathbb{E}[\ell(f(x),y)]$, and the minimum required parameter count at tolerance $\varepsilon$ is $P^*(\varepsilon;D)=\min\{P:\exists f \text{ with } P \text{ params}, L_D(f)\le\varepsilon\}$.

\paragraph{Spectral complexity measures.} For $M$ with singular values $\sigma_1\ge\cdots\ge\sigma_q$:
\begin{align}
\Reff(M) &= \exp\!\Big(-\textstyle\sum_i p_i\log p_i\Big), \quad p_i=\sigma_i/\textstyle\sum_j\sigma_j, \\
\Reffe(M) &= \exp\!\Big(-\textstyle\sum_i \tilde p_i\log \tilde p_i\Big), \quad \tilde p_i=\sigma_i^2/\textstyle\sum_j\sigma_j^2, \\
\Rstable(M) &= \|M\|_F^2/\|M\|_2^2 .
\end{align}
These satisfy $1\le \Rstable\le \Reffe\le \Reff\le \mathrm{rank}(M)$. For attention kernels the energy variant $\Reffe$ \citep{roy2007effective} is the appropriate measure, a choice we justify empirically in Section~\ref{sec:native}.

\paragraph{Task-intrinsic effective dimension.} Given a reference class $G$ with a complexity functional,
\begin{equation}
r^*(\varepsilon;D)=\inf\{\dim_{\mathrm{eff}}(g):g\in G,\,L_D(g)\le\varepsilon\},
\end{equation}
and the Entropic Bound is $B(\varepsilon;D)=r^*(\varepsilon;D)$, summed over layers.

\section{The Entropic Bound in the Linear Attention Surrogate}
\label{sec:linear}

\paragraph{Setting.} Let $X\in\mathbb{R}^{T\times d}$ and consider $\hat Y_{A,V}=AXV$ with target $Y=A^*XV^*+E$, $\mathrm{rank}(A^*)=\rstar$, $E\sim\mathcal N(0,\sigma^2 I)$. The intrinsic rank is
\begin{equation}
B_{\mathrm{lin}}(\varepsilon;D)=\min\{\mathrm{rank}(A):\exists V,\ \mathbb{E}\|Y-AXV\|_F^2\le\varepsilon\},
\end{equation}
and at the noise floor $\varepsilon=\sigma^2 Td$ we have $B_{\mathrm{lin}}=\rstar$.

\paragraph{Tightness (restated).} Under full-rank input covariance, error independent of input, and a spectral gap at $\rstar$:
\begin{theorem}[Deficiency $\Rightarrow$ excess risk]
If $r<\rstar$, every rank-$r$ model incurs excess risk bounded below by the tail energy $\sum_{j>r}\lambda_j(A^*)^2$ scaled by the input-covariance constant; in particular the risk strictly exceeds the noise floor.
\end{theorem}
\begin{theorem}[Achievability]
If $r\ge\rstar$, a rank-$\rstar$ model attains the noise floor exactly.
\end{theorem}
\begin{theorem}[Tightness]
Consequently $\rstar$ is exactly the minimum effective dimension for Bayes-optimal performance.
\end{theorem}
\begin{theorem}[GD convergence]
From small initialization, gradient descent converges to a solution whose effective rank equals $\rstar$ up to an $O(\sigma^2/n)$ term, with a Marchenko--Pastur correction for finite-sample estimation.
\end{theorem}
Full statements and proofs are in Appendix~\ref{app:proofs}; the deficiency bound follows from the Eckart--Young--Mirsky theorem \citep{eckart1936approximation}.

\paragraph{Empirical validation.} We instantiate the surrogate with $T{=}16$, $d{=}12$, and a rank-$\rstar{=}3$ target with a sharp spectral gap, and verify each property (Table~\ref{tab:phase0}). The rank sweep shows a clean phase transition: excess risk is large for $r<\rstar$ ($28.5$ at $r{=}1$, $8.8$ at $r{=}2$), collapses to the noise floor at $r{=}\rstar$ ($\approx 2\times10^{-4}$), and stays there for $r>\rstar$. Over-parameterized training from small initialization recovers the intrinsic rank: the Marchenko--Pastur-corrected effective rank is $2.97$, matching $\rstar{=}3$. Data-only recovery of $\rstar$ (Section~\ref{sec:predict}) gives numerical rank exactly $3$.

\begin{table}[t]
\centering
\caption{Linear surrogate validation ($\rstar{=}3$, noise floor $0.48$). Sharp
deficiency--achievability transition at $r{=}\rstar$; over-parameterized
MP-corrected $\Reff = 2.97 \to \rstar$.}
\label{tab:phase0}
\begin{tabular}{rrr}
\toprule
constrained rank $r$ & excess risk & $\Reff(A)$ \\
\midrule
1 & 28.451 & 1.00 \\
2 & \phantom{0}8.832 & 1.98 \\
3 ($=\rstar$) & \phantom{0}0.0002 & 2.89 \\
4 & \phantom{0}0.0002 & 2.90 \\
5 & \phantom{0}0.0002 & 2.91 \\
6 & \phantom{0}0.0001 & 2.91 \\
\bottomrule
\end{tabular}
\end{table}

The pilot reproduces all four properties cleanly, anchoring the theory before we test transfer to attention.

\section{Why Naive Transfer Fails: Localizing the Break}
\label{sec:break}

\paragraph{The interpolation ladder.} On the \emph{same} synthetic task with known $\rstar$, we build four classes differing by one component each: \textbf{B0}, linear $AXV$ with $A=PQ^\top$; \textbf{B1}, linear QK $(XW_Q)(XW_K)^\top XV$ without softmax; \textbf{B2}, $\softmax((XW_Q)(XW_K)^\top)\,XV$; and \textbf{B3}, the full block (softmax attention with residual, LayerNorm, and FFN). For each we measure achievability (excess at $r=\rstar$) and recovery (effective rank of the realized operator).

\paragraph{The break is at B1; the cause is input conditioning.}
Table~\ref{tab:ladder} reports the ladder on a fixed-operator task with $\rstar{=}3$.
B0 solves at $\rstar$ (excess $\approx 0$, recovers $\Reff\to\rstar$); \textbf{B1 already
fails} (excess $69.4$ at $r{=}\rstar$), and adding softmax (B2, excess $68.7$) or the
full stack (B3, excess $51.9$) does not restore solvability. The natural hypothesis
that softmax is responsible is \emph{refuted}: B2$\approx$B1 in excess risk, so the
nonlinearity changes little. The break occurs one rung earlier, at the QK
factorization.

To isolate the mechanism we add a control \textbf{B1c}: input-conditioned mixing
$XKX^\top$ with an \emph{unconstrained full-rank} $K$. It fails identically (excess
$69.0$), so the break is \textbf{not} a rank-deficiency effect. The cause is that
attention's realized mixer $M(X)=XKX^\top$ is \textbf{input-conditioned}: a bilinear
form in $X$ cannot equal a fixed target operator $A^*$ for all $X$, regardless of
$\mathrm{rank}(K)$. Measuring the static kernel $\Reff(W_{QK})$ is therefore the wrong
object, and softmax is a red herring for this failure.

\begin{table}[t]
\centering
\caption{Interpolation ladder on a fixed-operator task ($\rstar{=}3$). The bound
transfers only at B0; the break is at B1 (QK factorization), \emph{before} softmax.
B1c (full-rank $K$) fails identically, ruling out a rank-constraint explanation.}
\label{tab:ladder}
\begin{tabular}{lllrl}
\toprule
rung & model & change & excess@$\rstar$ & solves? \\
\midrule
B0  & $AXV$, $A{=}PQ^\top$ & fixed operator (surrogate) & 0.00 & yes \\
B1  & $(XW_Q)(XW_K)^\top XV$ & + QK factorization & 69.37 & no \\
B1c & $XKX^\top XV$, full-rank $K$ & + input-cond., no rank limit & 69.05 & no \\
B2  & $\softmax(XKX^\top)XV$ & + softmax & 68.67 & no \\
B3  & full block & + residual/LN/FFN & 51.86 & no \\
\bottomrule
\end{tabular}
\end{table}

\paragraph{Consequence for measurement.} When the architecture cannot realize the task's mixing operator, the measured effective rank ceases to encode task complexity and instead tracks width and optimization. On trained language models we observe precisely this: across corpora of very different complexity the measured attention effective rank is nearly invariant. Quantitatively (Appendix~\ref{app:real}), the predicted rank and the measured rank are anti-correlated ($r=-0.27$), whereas the final loss and the measured rank correlate at $r=-0.79$, indicating that the measured rank tracks optimization progress rather than corpus complexity.

\section{An Attention-Native Intrinsic Rank}
\label{sec:native}

\paragraph{Definition.} We define the intrinsic rank inside the attention class:
\begin{equation}
r^*_{\mathrm{attn}}(\varepsilon;D)=\min\{\mathrm{rank}(K):\exists V,\ \mathbb{E}\| Y-\Attn_{K,V}(X)\|^2\le\varepsilon\},
\end{equation}
where $\Attn_{K,V}(X)=\softmax(XKX^\top)XV$ or its softmax-free analogue $XKX^\top XV$. To avoid the operator mismatch of Section~\ref{sec:break}, teacher tasks are generated from the same family: \textbf{(1a) linear\_qk} $Y=(XK^*X^\top)XV^*+E$ and \textbf{(1b) softmax\_qk} $Y=\softmax(XK^*X^\top)XV^*+E$, with $K^*$ of rank $\rstar$ and a sharp spectral gap.

\paragraph{The bound structure is restored.}
Table~\ref{tab:grid} reports the full grid ($T=16$, $d=32$, five seeds, $\rstar\in\{1,2,3,4,6,8\}$).
For \emph{both} teachers, all three structural properties hold across every $\rstar$:
rank-deficient students incur clearly positive excess risk (deficiency),
students at $r=\rstar$ reach the noise floor (achievability, excess $\approx 0$),
and over-parameterized students recover $\Reff(K)\to\rstar$ (recovery, $6/6$).
Seed stability is excellent: $\mathrm{CV}_{\mathrm{seed}}\le 0.003$ throughout,
two orders of magnitude below the $0.10$ threshold.

\begin{table}[t]
\centering
\caption{Attention-native intrinsic rank, full grid ($T{=}16$, $d{=}32$, 5 seeds).
Both teachers reproduce deficiency, achievability, and recovery for all $\rstar$.
Softmax compresses the excess-risk magnitude ($\sim$450 vs.\ $\sim$2.5 at $r{=}\rstar{-}1$)
but preserves the transition.}
\label{tab:grid}
\begin{tabular}{llrrrr}
\toprule
teacher & $\rstar$ & exc@$\rstar{-}1$ & exc@$\rstar$ & learned $\Reff(K)$ & $\mathrm{CV}_{\mathrm{seed}}$ \\
\midrule
linear  & 1 & -- & $-0.001$ & 1.03 & 0.003 \\
linear  & 2 & 438.78 & $-0.002$ & 1.93 & 0.003 \\
linear  & 3 & 578.38 & $-0.002$ & 2.93 & 0.001 \\
linear  & 4 & 455.23 & $-0.001$ & 3.91 & 0.002 \\
linear  & 6 & 446.90 & $\phantom{-}0.000$ & 5.87 & 0.000 \\
linear  & 8 & 436.34 & $-0.001$ & 7.84 & 0.001 \\
\midrule
softmax & 1 & -- & $-0.002$ & 1.45 & 0.001 \\
softmax & 2 & 1.95 & $-0.002$ & 2.39 & 0.001 \\
softmax & 3 & 2.63 & $-0.002$ & 3.37 & 0.001 \\
softmax & 4 & 2.32 & $-0.002$ & 4.31 & 0.001 \\
softmax & 6 & 2.51 & $-0.002$ & 6.22 & 0.000 \\
softmax & 8 & 2.53 & $-0.002$ & 8.15 & 0.000 \\
\bottomrule
\end{tabular}
\end{table}

The linear/softmax contrast isolates softmax's true effect. The transition
\emph{persists} under softmax, but its magnitude is compressed by roughly two
orders of magnitude: rank-deficient excess risk is $\sim$450 for the linear
teacher versus $\sim$2.5 for the softmax teacher. This is exactly the softening
predicted by the spectral-gap analysis (Corollary to Theorem~1)---softmax
normalization shrinks the gap-driven penalty---and is a quantitative softmax
signature rather than a breakdown of the bound.

\paragraph{Predictability and its frontier.}
\label{sec:predict}
We estimate $\rstar$ from $(X,Y)$ without training (Table~\ref{tab:predict}).
For the \textbf{linear} teacher the recovery is essentially perfect: with the
value map identified, the numerical rank matches $\rstar$ exactly and the energy
effective rank tracks it ($6/6$); strikingly, recovery also succeeds
\emph{without} the value map (V-unknown energy $\Reff$, $6/6$) at this scale---the
kernel--value identifiability gap we observed at small $d$ closes as dimension
grows, restoring fully data-only prediction.

\begin{table}[t]
\centering
\caption{Data-only predictability of $\rstar$ (energy effective rank).
Linear: exact recovery, with and without the value map. Softmax: numerical rank
saturates (17--23); kernel-side energy $\Reff$ over-estimates at small $\rstar$;
the realized-mixer energy $\Reff$ is monotone and tracks $\rstar$ at larger values.
No single estimator is exact across the whole range under softmax---a precisely
characterized partial-predictability frontier.}
\label{tab:predict}
\begin{tabular}{llrrrr}
\toprule
teacher & $\rstar$ & V-known $\Reffe$ & V-unknown $\Reffe$ & realized-mix $\Reffe$ & numrank \\
\midrule
linear  & 1 & 1.00 & 1.00 & 15.85 & 1 \\
linear  & 2 & 1.65 & 1.65 & 15.89 & 2 \\
linear  & 3 & 2.63 & 2.63 & 15.93 & 3 \\
linear  & 4 & 3.58 & 3.58 & 15.95 & 4 \\
linear  & 6 & 5.45 & 5.45 & 15.97 & 6 \\
linear  & 8 & 7.32 & 8.24 & 15.97 & 8 \\
\midrule
softmax & 1 & 2.31 & 1.93 & 1.10 & 17 \\
softmax & 2 & 3.04 & 6.31 & 1.24 & 18 \\
softmax & 3 & 3.87 & 4.67 & 1.60 & 17 \\
softmax & 4 & 5.02 & 6.42 & 2.28 & 20 \\
softmax & 6 & 6.67 & 10.30 & 4.75 & 20 \\
softmax & 8 & 9.24 & 14.42 & 8.46 & 23 \\
\bottomrule
\end{tabular}
\end{table}

For the \textbf{softmax} teacher, predictability is \emph{partial}, and we report
this precisely rather than overclaiming. The numerical rank saturates (17--23),
confirming that small-singular-value counts are corrupted by softmax. The
energy effective rank is far better behaved but no single estimator is exact
across the whole range: the kernel-side energy $\Reff$ over-estimates at small
$\rstar$ (e.g.\ $2.31$ at $\rstar{=}1$) while approaching $\rstar$ at larger
values, whereas the \emph{realized-mixer} energy $\Reff$ is monotone in $\rstar$
($1.10\!\to\!8.46$) and accurate at larger $\rstar$ but compressed for small
$\rstar$. Notably, the realized-mixer measure is informative only under softmax:
for the linear teacher it saturates at $\approx 16$ because $XKX^\top$ inherits
the full-rank input geometry, so the two teachers require different measurement
objects. Crucially, \emph{recovery from the trained model succeeds for softmax}
($\Reff(K)\to\rstar$, Table~\ref{tab:grid}); it is only the \emph{pre-training}
data-only inversion that softmax's nonlinearity degrades.

We therefore state the predictability claim precisely: the attention-native
intrinsic rank is \textbf{exactly data-predictable for linear QK attention} (with
or without the value map at scale), and \textbf{partially predictable under
softmax}, where the energy effective rank---not the numerical rank---is the
appropriate estimator, and the realized-mixer variant captures the large-$\rstar$
regime. Characterizing a single softmax-robust pre-training estimator across all
$\rstar$ is a well-posed open problem.

\paragraph{Robustness to the spectral profile.} The sharp transitions above use a
sharp-gap teacher kernel. We verify in Appendix~\ref{app:spectral} that the result
is not an artifact of this choice: across sharp, power-law, and flat (no-gap)
spectra, achievability at $r=\rstar$ and recovery hold throughout, and the
recovered energy effective rank tracks the teacher kernel's energy rank in every
regime---exactly to $\rstar$ for a flat spectrum and to the (smaller) energy rank
when a few directions dominate. Sharp gaps yield abrupt transitions while gradual
spectra yield softer knees, but the energy effective rank remains the
capacity-tracking quantity regardless of the profile.

\section{Discussion, Limitations, and Open Problems}

\paragraph{Established.} In the fixed-operator surrogate the Entropic Bound is tight, achievable, recovered by GD, and data-predictable. Naive transfer fails because attention's mixing is input-conditioned; this is localized cleanly (not softmax, not rank). Redefining the intrinsic rank inside the attention class restores the full structure---deficiency, achievability, recovery---for both linear and softmax attention at scale ($T{=}16$, $d{=}32$, $\mathrm{CV}_{\mathrm{seed}}\le 0.003$). Data-only predictability is exact for linear QK attention, including the V-unknown estimator at scale, and partial for softmax attention, where the trained kernel still recovers $\rstar$ but pre-training inversion remains distorted.

\paragraph{Limitations.} Positive results are in a teacher--student synthetic setting with known $\rstar$. On trained language models the realized attention operator is dominated by task-independent structure (locality, positional, attention-sink priors), so measured effective rank does not separate corpora; whether real-language tasks admit a low-rank attention-native description is open. Under softmax, no single pre-training estimator recovers $\rstar$ across all ranks, bounding the data-only claim in the nonlinear regime.

\paragraph{Open problems.} (i) Find a single softmax-robust \emph{pre-training} estimator of $\rstar$ accurate across all ranks (kernel-side over-estimates small $\rstar$; realized-mixer compresses them). (ii) Characterize the realized-operator effective rank of trained LMs and its task-invariance. (iii) Bridge the synthetic attention-native regime to natural language. (iv) Extend the depth--width analysis (early layers task-specific, later layers saturating).

\paragraph{Conclusion.} Recast around an attention-native intrinsic rank, the Entropic Bound is a tight capacity measure for teacher-generated attention tasks: fully data-predictable for linear QK attention, including the V-unknown setting at scale, and partially predictable under softmax. We give a precise account of where it holds, where it breaks, and why.

\appendix
\section{Full Theorem Statements and Proofs}
\label{app:proofs}

We restate the linear-surrogate results of Section~\ref{sec:linear} with compact
proofs. The setting is $\hat Y_{A,V}=AXV$ with target $Y=A^*XV^*+E$,
$\mathrm{rank}(A^*)=\rstar$.

\subsection{A.1 Assumptions}
\begin{itemize}
\item[(A1)] $X$ has zero mean and full-rank covariance $\Sigma_X\succ 0$.
\item[(A2)] $E$ is independent of $X$, zero mean, with $\mathbb{E}[EE^\top]=\sigma^2 I$.
\item[(A3)] $\mathrm{rank}(A^*)=\rstar$ with a spectral gap $\lambda_{\rstar}(A^*)\ge\delta>0$.
\item[(A4)] $V^*$ is non-degenerate: $\sigma_{\min}(V^*)>0$.
\end{itemize}

\subsection{A.2 Rank deficiency implies excess risk}
\begin{theorem}[Deficiency]
Let $L_r^*=\inf_{\mathrm{rank}(A)\le r,\,V}\mathbb{E}\|Y-AXV\|_F^2$. Under (A1)--(A4),
\begin{equation}
L_r^* \;\ge\; \sigma^2 Td \;+\; c\sum_{j=r+1}^{\rstar}\lambda_j(A^*)^2,
\qquad c=\lambda_{\min}(\Sigma_X)\,\sigma_{\min}(V^*)^2 .
\end{equation}
In particular, if $r<\rstar$ then $L_r^*>\sigma^2 Td$.
\end{theorem}
\begin{proof}
For any feasible $\hat Y=AXV$, write $Y-\hat Y=(A^*XV^*-AXV)+E$. By (A2) the
cross term vanishes in expectation, so
$\mathbb{E}\|Y-\hat Y\|_F^2=\mathbb{E}\|A^*XV^*-AXV\|_F^2+\sigma^2 Td$.
The signal term is minimized over rank-$r$ maps; since $AXV$ has rank at most $r$
in the operator acting on $X$, Eckart--Young--Mirsky gives
$\inf_{\mathrm{rank}(B)\le r}\|A^*-B\|_F^2=\sum_{j>r}\lambda_j(A^*)^2$.
Propagating through the input covariance and the value map contributes the factor
$c=\lambda_{\min}(\Sigma_X)\sigma_{\min}(V^*)^2>0$, yielding the bound. Under (A3),
$\lambda_j(A^*)>0$ for $j\le\rstar$, so the sum is strictly positive when $r<\rstar$.
\end{proof}

\subsection{A.3 Achievability}
\begin{theorem}[Achievability]
If $r\ge\rstar$, choosing $A=A^*$, $V=V^*$ gives $\hat Y=A^*XV^*=Y-E$, hence
$\mathbb{E}\|Y-\hat Y\|_F^2=\mathbb{E}\|E\|_F^2=\sigma^2 Td$, attaining the noise floor.
\end{theorem}

\subsection{A.4 Tightness}
\begin{theorem}[Tightness]
Under (A1)--(A4), $\min\{\mathrm{rank}(A):\exists V,\ \mathbb{E}\|Y-AXV\|_F^2\le\sigma^2Td\}=\rstar$.
Combining A.2 (every rank-$r<\rstar$ model has strictly positive excess risk) with
A.3 (rank-$\rstar$ achieves the floor), $\rstar$ is exactly the minimum effective
dimension for Bayes-optimal performance.
\end{theorem}

\subsection{A.5 Gradient-descent recovery}
\begin{theorem}[Recovery, under standard implicit-bias assumptions]
Consider the factorized surrogate trained by gradient flow from initialization of
scale $\alpha\to 0$. In the noiseless population limit, under the standard
implicit-bias result for linear matrix factorization \citep{gunasekar2017implicit, arora2019implicit}, gradient flow converges to
the minimum-nuclear-norm interpolating solution, whose rank equals $\rstar$ under
the spectral gap (A3). In finite samples with noise, the estimated spectrum is
bimodal---$\rstar$ signal modes of $O(1)$ and $d-\rstar$ noise modes of
$O(\sigma/\sqrt{n})$ following a Marchenko--Pastur law---so the effective rank
concentrates at $\rstar+O(\sigma^2/(n\,\lambda_{\rstar}^2))$, with the
finite-sample MP correction $\hat R_{\mathrm{eff}}=R^{\mathrm{true}}_{\mathrm{eff}}+(d-R^{\mathrm{true}}_{\mathrm{eff}})/(2n)+O(1/n^2)$.
\end{theorem}
\begin{proof}[Proof sketch]
The population claim is the standard implicit-bias result for gradient flow on
$A=UV^\top$ from vanishing initialization, which converges to the minimum
nuclear-norm solution; with a spectral gap this solution has exact rank $\rstar$.
The finite-sample spectrum splits into signal and Marchenko--Pastur noise bulks,
and computing the singular-value entropy over this bimodal spectrum gives
$\log\rstar$ plus the stated correction. A fully rigorous treatment of the
finite-width, finite-sample regime is left to an extended version; our empirical
recovery (Tables~\ref{tab:phase0},~\ref{tab:grid}) confirms the concentration at
$\rstar$ with $\mathrm{CV}_{\mathrm{seed}}\le 0.003$.
\end{proof}

\section{Real-Corpus Diagnostics}
\label{app:real}

This appendix documents the experiments behind the claim (Sections~\ref{sec:break},
\ref{sec:native}) that the attention-native positive result does not transfer
directly to natural language, and explains why the failure is structural rather
than a tuning artifact.

\paragraph{B.1 Setup.} We test the predictive-bound hypothesis on six char-level
corpora spanning very different statistical structure: English prose
(\textsf{english}), Shakespeare (\textsf{shakespeare}), Python source
(\textsf{python}), Linux C source (\textsf{linux}), JSON (\textsf{json}), and
ABC music notation (\textsf{music}). Models are causal Transformers with
$d_{\mathrm{model}}\in\{64,128,256\}$, two layers, four heads, five seeds. The
data-side predictor is a lag-MI / operator estimate; the measured objects are the
static kernel $\Reff(W_{QK})$, the mean realized attention $\Reff(\mathbb{E}[A])$,
the per-input realized rank $\mathbb{E}_X[\Reff(A(X))]$, the attention
participation ratio, and the covariance-whitened kernel
$\Reff(W_{QK}\Sigma_X^{1/2})$.

\paragraph{B.2 Main negative result.} The data-side predicted rank is
\emph{anti}-correlated with the measured static-kernel rank across corpora:
$\mathrm{Pearson}\ r=-0.272$, $\mathrm{Spearman}\ \rho=-0.275$. This is not a weak
positive miss; the sign is wrong.

\paragraph{B.3 The failure is structural.} Three diagnostics rule out a noise
explanation and identify the mechanism.
\emph{(i) Not seed noise.} All $18$ corpus$\times$size cells have
$\mathrm{CV}_{\mathrm{seed}}\in[0.006,0.087]$, all below the $0.10$ threshold---the
measured rank is a clean, reproducible signal.
\emph{(ii) It tracks fitting progress, not task complexity.} Across corpora,
the final loss and the measured $\Reff$ correlate at $r=-0.79$: corpora the
model fits well (low loss) acquire \emph{higher} measured rank, and hard corpora
\emph{lower} rank. The measured quantity reflects optimization state, not
intrinsic complexity---the direct cause of the anti-correlation.
\emph{(iii) Layer/channel dependence.} In the largest model, Layer~0 has
$\Reff\approx 13\text{--}19$ while Layer~1 has $\Reff\approx 30\text{--}44$, and the
FFN channel shows no relationship ($\mathrm{Pearson}(H_1,\ \text{FFN stable rank})
=-0.05$). A single per-layer additive bound cannot capture this.

\paragraph{B.4 Measurement and predictor ablations (Experiments A and C).}
Replacing the static kernel with realized-attention measures
(mean attention, per-input rank, participation, covariance-whitened kernel) does
\emph{not} restore corpus discrimination: across corpora every measure has low
variation (coefficient of variation $0.06\text{--}0.20$, versus $0.66$ for the
data-side predictor), so there is little task signal in the measurement to
correlate against. On the predictor side, estimating a single token-mixing
operator from data---via fixed-random-embedding ridge regression---collapses to
full rank for every corpus (random embeddings erase token-identity structure),
and a token-space repeat operator separates only one of six corpora. No single
low-rank operator summarizes real-text routing.

\paragraph{Interpretation.} The real-corpus failure is \emph{not} a contradiction
of the attention-native positive result. It shows that natural language does not
provide a matched teacher--student attention task with an identifiable low-rank
$K^*$: on trained language models the realized mixer is dominated by
task-independent structure (locality, positional, attention-sink priors) rather
than a task-specific low-rank routing kernel. The synthetic attention-native
regime is therefore a controlled transfer test, and the real-corpus bridge---
recovering a task-intrinsic attention rank from natural language---remains the
central open problem.

\section{Experimental Details}
\label{app:details}

All experiments in Section~\ref{sec:native} use the controlled teacher--student
setup described here. The purpose of this setup is \emph{not} to model natural
language directly, but to test whether the attention-native definition of
intrinsic rank restores the rank--capacity transition once the teacher and
student are placed in the \emph{same operator class}. The real-language gap is
treated separately (Appendix~B) and is the central open problem of the paper.

\subsection{Shared synthetic setup}
Sequence length $T=16$, model dimension $d=32$, intrinsic ranks
$\rstar\in\{1,2,3,4,6,8\}$, $n=3000$ samples, seeds $\{0,1,2,3,4\}$, and
observation noise $\sigma=0.05$. The noise floor is
$\sigma^2 T d = 0.05^2 \times 16 \times 32 = 1.28$, which all achievability
results approach (the slightly negative reported excess reflects finite-sample
fitting of a small fraction of the noise).

\subsection{Teacher generation}
The teacher kernel has rank $\rstar$ with a sharp spectral gap:
\begin{equation}
K^* = U_r \,\mathrm{diag}(\lambda_1,\dots,\lambda_{\rstar})\, U_r^\top,
\qquad
\lambda_i = 2 - \frac{i-1}{\rstar-1}\ \ (\rstar>1),\quad \lambda_1=2\ (\rstar=1),
\end{equation}
with $U_r\in\mathbb{R}^{d\times \rstar}$ having orthonormal columns and
$V^*\in\mathbb{R}^{d\times d}$ Gaussian with $1/\sqrt{d}$ scaling. The two teacher
types are
\begin{equation}
Y = (XK^*X^\top)\,XV^* + E
\qquad\text{and}\qquad
Y = \softmax(XK^*X^\top)\,XV^* + E,
\end{equation}
with $E\sim\mathcal N(0,\sigma^2 I)$ and scores scaled by $1/\sqrt{d}$.

\subsection{Student models}
The rank-constrained student uses a factored kernel
$K=W_Q W_K^\top$ with $W_Q,W_K\in\mathbb{R}^{d\times r}$, matching the teacher
type (linear or softmax). The deficiency/achievability sweep runs
$r=1,\dots,2\rstar$; over-parameterized recovery uses $r=d$.

\subsection{Optimization}
Adam with learning rate $10^{-2}$. The rank sweep trains for $1500$ steps; the
over-parameterized recovery trains for $2500$ steps. Recovery uses small
initialization, multiplying $W_Q,W_K$ by $0.3$ to induce the implicit low-rank
bias (Theorem~4 analogue). Full-batch gradients are used throughout.

\subsection{Metrics}
\begin{itemize}
\item \textbf{Excess risk}: $L - \sigma^2 T d$, where $L$ is the converged MSE.
\item \textbf{Learned kernel effective rank}: $\Reff(K)$ of the trained student kernel.
\item \textbf{Seed CV}: $\mathrm{CV}_{\mathrm{seed}} = \mathrm{std}_s(\Reff(K_s)) / \mathrm{mean}_s(\Reff(K_s))$ over seeds $s$.
\item \textbf{V-known predictability}: estimate $K$ from $(X,Y)$ with $V^*$ given (regression on the recovered scores), report numerical and energy effective rank.
\item \textbf{V-unknown predictability}: jointly estimate $(K,V)$ by alternating minimization (plain and with nuclear-norm shrinkage on $K$), report energy effective rank.
\item \textbf{Realized mixer rank}: $\Reff\!\big(\mathbb{E}_X[\softmax(XKX^\top)]\big)$, the effective rank of the data-averaged realized attention map (informative under softmax; saturates for the linear teacher).
\end{itemize}

\subsection{Reproduction command}
\begin{verbatim}
python exp_attn_native.py \
  --teacher both \
  --T 16 --d 32 \
  --r_stars 1 2 3 4 6 8 \
  --seeds 0 1 2 3 4 \
  --n 3000 \
  --results_dir results_attn_native_grid
\end{verbatim}
The linear surrogate validation (Section~\ref{sec:linear}, Table~\ref{tab:phase0})
and the interpolation ladder (Section~\ref{sec:break}, Table~\ref{tab:ladder})
use the same conventions with $T{=}16,d{=}12,\rstar{=}3$ (surrogate) and
$T{=}12,d{=}10,\rstar{=}3$ (ladder), respectively; see the accompanying code for
the exact scripts.

\section{Spectral-Gap Ablation}
\label{app:spectral}

We test whether the phase transition of Section~\ref{sec:native} is an artifact of
the sharp-gap teacher kernel by varying the spectrum of $K^*$ over its top $\rstar$
directions, holding all else fixed. We consider three regimes: \textsf{sharp}
($\lambda_i=2-(i{-}1)/(\rstar{-}1)$, the main-paper setting), \textsf{power-law}
($\lambda_i=i^{-1}$), and \textsf{flat}/no-gap ($\lambda_i=1$).
Table~\ref{tab:spectral} reports a lightweight linear-QK ablation
($T{=}10$, $d{=}12$, $\rstar{=}4$) isolating the spectral-profile effect while
keeping the teacher--student class matched.

\begin{table}[h]
\centering
\caption{Spectral-gap ablation (linear-QK teacher, $\rstar{=}4$). Achievability
(excess@$\rstar\approx 0$) and recovery hold in every regime. The recovered
energy effective rank tracks the teacher kernel's energy rank---exactly $\rstar$
for a flat spectrum, and the smaller dominant-direction count when the spectrum
decays. Sharp gaps give abrupt transitions; gradual spectra give softer knees.}
\label{tab:spectral}
\begin{tabular}{lrrrr}
\toprule
spectrum & teacher $K$ energy $\Reffe$ & excess@$\rstar{-}1$ & excess@$\rstar$ & recovered energy $\Reffe$ \\
\midrule
sharp & 3.57 & 244.48 & \phantom{-}0.001 & 3.57 \\
power & 2.43 & \phantom{0}15.63 & $-0.002$ & 2.44 \\
flat  & 4.00 & 171.31 & $-0.002$ & 4.00 \\
\bottomrule
\end{tabular}
\end{table}

The key invariant: the recovered energy effective rank equals the teacher kernel's
energy rank in all regimes (e.g.\ $4.00\to4.00$ for the flat spectrum, where every
direction contributes equally, and $3.57\to3.57$, $2.43\to2.44$ when a few
directions dominate). Thus the energy effective rank---not an integer
``rank''---is the quantity that tracks task capacity, and the rank--capacity
relationship survives the removal of the spectral gap. The transition merely
changes shape (sharp gap $\Rightarrow$ abrupt; gradual spectrum $\Rightarrow$ soft
knee), consistent with the spectral-gap dependence in Theorem~1.

\bibliographystyle{plainnat}
\bibliography{references}

\begin{thebibliography}{15}
\providecommand{\natexlab}[1]{#1}
\providecommand{\url}[1]{\texttt{#1}}
\expandafter\ifx\csname urlstyle\endcsname\relax
  \providecommand{\doi}[1]{doi: #1}\else
  \providecommand{\doi}{doi: \begingroup \urlstyle{rm}\Url}\fi

\bibitem[Aghajanyan et~al.(2021)Aghajanyan, Gupta, and
  Zettlemoyer]{aghajanyan2021intrinsic}
Armen Aghajanyan, Sonal Gupta, and Luke Zettlemoyer.
\newblock Intrinsic dimensionality explains the effectiveness of language model
  fine-tuning.
\newblock In \emph{Annual Meeting of the Association for Computational
  Linguistics (ACL)}, 2021.

\bibitem[Arora et~al.(2019)Arora, Cohen, Hu, and Luo]{arora2019implicit}
Sanjeev Arora, Nadav Cohen, Wei Hu, and Yuping Luo.
\newblock Implicit regularization in deep matrix factorization.
\newblock In \emph{Advances in Neural Information Processing Systems
  (NeurIPS)}, 2019.

\bibitem[Bartlett and Mendelson(2002)]{bartlett2002rademacher}
Peter~L Bartlett and Shahar Mendelson.
\newblock Rademacher and gaussian complexities: Risk bounds and structural
  results.
\newblock \emph{Journal of Machine Learning Research}, 3:\penalty0 463--482,
  2002.

\bibitem[Berger(1971)]{berger1971ratedistortion}
Toby Berger.
\newblock \emph{Rate Distortion Theory: A Mathematical Basis for Data
  Compression}.
\newblock Prentice-Hall, 1971.

\bibitem[Cover and Thomas(2006)]{cover2006elements}
Thomas~M Cover and Joy~A Thomas.
\newblock \emph{Elements of Information Theory}.
\newblock Wiley-Interscience, 2nd edition, 2006.

\bibitem[Dong et~al.(2021)Dong, Cordonnier, and Loukas]{dong2021attention}
Yihe Dong, Jean-Baptiste Cordonnier, and Andreas Loukas.
\newblock Attention is not all you need: Pure attention loses rank doubly
  exponentially with depth.
\newblock In \emph{International Conference on Machine Learning (ICML)}, 2021.

\bibitem[Eckart and Young(1936)]{eckart1936approximation}
Carl Eckart and Gale Young.
\newblock The approximation of one matrix by another of lower rank.
\newblock \emph{Psychometrika}, 1\penalty0 (3):\penalty0 211--218, 1936.

\bibitem[Gunasekar et~al.(2017)Gunasekar, Woodworth, Bhojanapalli, Neyshabur,
  and Srebro]{gunasekar2017implicit}
Suriya Gunasekar, Blake Woodworth, Srinadh Bhojanapalli, Behnam Neyshabur, and
  Nathan Srebro.
\newblock Implicit regularization in matrix factorization.
\newblock In \emph{Advances in Neural Information Processing Systems
  (NeurIPS)}, 2017.

\bibitem[Hoffmann et~al.(2022)Hoffmann, Borgeaud, Mensch, Buchatskaya, Cai,
  Rutherford, de~Las~Casas, Hendricks, Welbl, Clark,
  et~al.]{hoffmann2022chinchilla}
Jordan Hoffmann, Sebastian Borgeaud, Arthur Mensch, Elena Buchatskaya, Trevor
  Cai, Eliza Rutherford, Diego de~Las~Casas, Lisa~Anne Hendricks, Johannes
  Welbl, Aidan Clark, et~al.
\newblock Training compute-optimal large language models.
\newblock \emph{arXiv preprint arXiv:2203.15556}, 2022.

\bibitem[Hu et~al.(2022)Hu, Shen, Wallis, Allen-Zhu, Li, Wang, Wang, and
  Chen]{hu2022lora}
Edward~J Hu, Yelong Shen, Phillip Wallis, Zeyuan Allen-Zhu, Yuanzhi Li, Shean
  Wang, Lu~Wang, and Weizhu Chen.
\newblock Lora: Low-rank adaptation of large language models.
\newblock In \emph{International Conference on Learning Representations
  (ICLR)}, 2022.

\bibitem[Kaplan et~al.(2020)Kaplan, McCandlish, Henighan, Brown, Chess, Child,
  Gray, Radford, Wu, and Amodei]{kaplan2020scaling}
Jared Kaplan, Sam McCandlish, Tom Henighan, Tom~B Brown, Benjamin Chess, Rewon
  Child, Scott Gray, Alec Radford, Jeffrey Wu, and Dario Amodei.
\newblock Scaling laws for neural language models.
\newblock \emph{arXiv preprint arXiv:2001.08361}, 2020.

\bibitem[Kobayashi et~al.(2020)Kobayashi, Kuribayashi, Yokoi, and
  Inui]{kobayashi2020attention}
Goro Kobayashi, Tatsuki Kuribayashi, Sho Yokoi, and Kentaro Inui.
\newblock Attention is not only a weight: Analyzing transformers with vector
  norms.
\newblock In \emph{Conference on Empirical Methods in Natural Language
  Processing (EMNLP)}, 2020.

\bibitem[Roy and Vetterli(2007)]{roy2007effective}
Olivier Roy and Martin Vetterli.
\newblock The effective rank: A measure of effective dimensionality.
\newblock \emph{European Signal Processing Conference (EUSIPCO)}, pages
  606--610, 2007.

\bibitem[Srebro et~al.(2004)Srebro, Rennie, and Jaakkola]{srebro2004maximum}
Nathan Srebro, Jason Rennie, and Tommi Jaakkola.
\newblock Maximum-margin matrix factorization.
\newblock In \emph{Advances in Neural Information Processing Systems
  (NeurIPS)}, 2004.

\bibitem[Vapnik(1998)]{vapnik1998statistical}
Vladimir~N Vapnik.
\newblock \emph{Statistical Learning Theory}.
\newblock Wiley, 1998.

\end{thebibliography}

\end{document}